\documentclass[letterpaper, 10 pt, conference]{ieeeconf}  % Comment this line out if you need a4paper

\IEEEoverridecommandlockouts                              % This command is only needed if 
\usepackage{amsmath,amsfonts}
\usepackage[caption=false]{subfig}
\usepackage{stfloats}
\usepackage{url}
\usepackage{graphicx}
\usepackage{siunitx}
\usepackage{xcolor}

\usepackage[subpreambles=true,print=false]{standalone}

\newtheorem{theorem}{Theorem}
\title{\LARGE \bf
Towards Decentralized Formation of Minimum-Length Communication Networks Using Robot Swarms
}

\author{
 Genki Miyauchi$^{1}$, Mohamed S. Talamali$^{2}$, Julian Rau$^{3}$ and Roderich Gro\ss$^{2,3}$%
\thanks{$^{1}$Genki Miyauchi is with the Bristol Robotics Laboratory, University of Bristol, Bristol, UK.
{\tt\small g.miyauchi@bristol.ac.uk}}%
\thanks{$^{2}$Mohamed S. Talamali and Roderich Gro{\ss} are with the School of Electrical and Electronic Engineering, University of Sheffield, Sheffield, UK. {\tt\small m.s.talamali@sheffield.ac.uk}}%
\thanks{$^{3}$Julian Rau and Roderich Gro{\ss} are with the Department of Computer Science, Technical University of Darmstadt, Darmstadt, Germany. {\tt\small \{julian.rau, roderich.gross\}@tu-darmstadt.de}}%
}

\begin{document}

\maketitle
\thispagestyle{empty}
\pagestyle{empty}

%%%%%%%%%%%%%%%%%%%%%%%%%%%%%%%%%%%%%%%%%%%%%%%%%%%%%%%%%%%%%%%%%%%%%%%%%%%%%%%%
\begin{abstract}

Multi-robot missions in infrastructure-denied environments frequently rely on reliable communication links between spatially separated locations. We propose a fully decentralized framework for constructing and dynamically maintaining communication networks without centralized topology planning or global positioning infrastructure. Driven strictly by local interactions, robots reconfigure local network topologies and adjust their physical positions to minimize overall network length while adhering to communication constraints. Formal analysis shows that our local reconfiguration operations guarantee continuous network connectivity, strictly decrease network length with every branch transfer, and bound worst-case performance to the shortest starlike tree. Embodied simulations and physical multi-robot experiments confirm that our approach forms networks near the length of centrally computed Euclidean Steiner trees. Additionally, the system dynamically adapts to moving targets and optimizes deployment by utilizing only necessary connectors, preserving excess robots for auxiliary tasks. This work enables autonomous swarms to self-organize adaptive ad hoc communication infrastructure in communication-denied environments, which could support applications varying from subterranean exploration to planetary missions.

\end{abstract}

%%%%%%%%%%%%%%%%%%%%%%%%%%%%%%%%%%%%%%%%%%%%%%%%%%%%%%%%%%%%%%%%%%%%%%%%%%%%%%%%
\section{Introduction}
\label{sec:introduction}

Multi-robot communication networks can support the exchange of task-relevant information and coordination among spatially separated robots, particularly in remote environments where communication infrastructure and positioning systems may be unavailable or unreliable. To maintain such networks over large distances, some robots can act as relay nodes, hereafter called connectors, by positioning themselves between distant robots and maintaining communication links with neighboring robots (see Fig.~\ref{fig:overview-screenshot}).
The remaining robots can then be used to perform other tasks. 

Because robots have limited communication range, networks of excessive lengths generally require more connectors, which can consume a substantial fraction of the robot team.
They can also increase the distance traveled by robots navigating or transporting goods along the network, as well as the number of relay hops required for communication.
Thus, beyond maintaining connectivity, reducing the total length of the communication network can improve how efficiently it supports the robot team.

When only two locations of interest must be connected, the minimal-length communication network is a path graph. For three or more locations, minimal-length networks take the form of tree structures known as the Euclidean Steiner trees~\cite{promel2002Steiner}. However, directly computing and deploying a Steiner tree is not realistic for the settings considered here: the robots do not have access to a shared global coordinate frame, knowledge of all robot positions, or a central authority that assigns network positions and connections. Moreover, the robots have limited communication range, and connectivity must be maintained as the robots move.

Existing approaches to multi-robot connectivity maintenance have demonstrated robot teams that preserve communication links~\cite{krupke2015distributed}, form chains or trees~\cite{Majcherczyk2018, varadharajan2020swarm}, and adapt to changing task locations~\cite{luo2020behavior, miyauchi2022multi}. Nevertheless, many decentralized approaches primarily focus on maintaining connectivity rather than reducing the total length of the resulting network. Therefore, the placement of connectors may be inefficient, causing the network to use more robots than necessary. Centralized approaches can optimize relay placement or network structure directly~\cite{lee2022dynamic}, but they typically require global position information and centralized assignment of robot behaviors. 

\begin{figure}
    \centering
    \includegraphics[width=\linewidth]{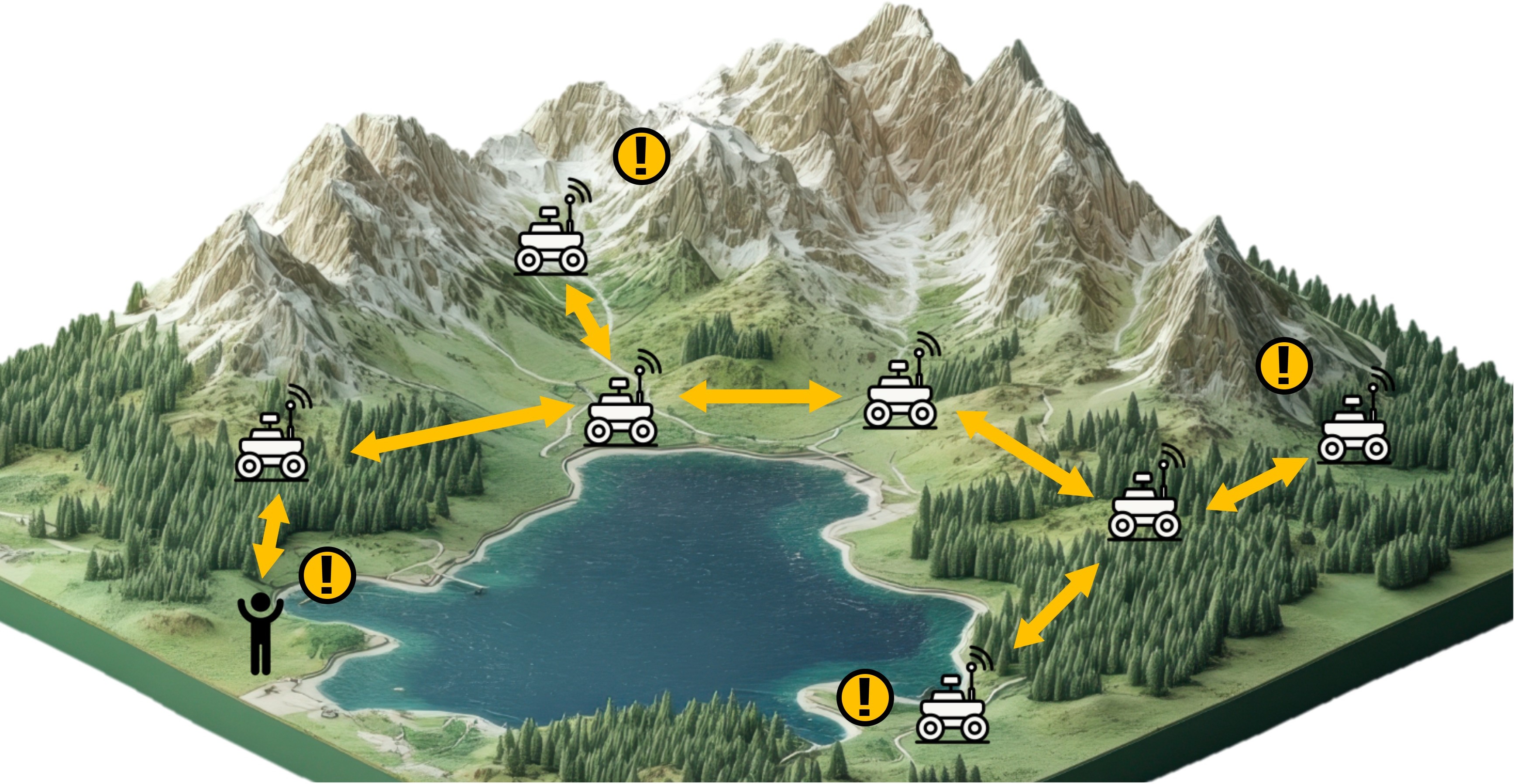}
    \caption{Robots forming a communication network between several locations in the environment. Some robots (or humans) operate at the locations of interest (indicated by exclamation marks) while other robots maintain the network in an environment where global positioning systems and communication infrastructure are not available.}
    \label{fig:overview-screenshot}
\end{figure}

This paper proposes a decentralized controller that enables a group of mobile robots to form and maintain short, Steiner-tree-like communication networks.
The robots do not rely on global positioning, global communication, or a central planner. Instead, they use local sensing and communication to decide when to add connectors to maintain connectivity, when redundant connectors can be removed, and when network connections should be reassigned to nearby connectors to reduce total edge length. In addition, each connector locally adjusts its position relative to its network neighbors to make the communication network approach Steiner-tree-like structures. 
We show that the network remains connected at all times, stabilizes to a shortest starlike tree without merging redundant edges, and can be further shortened by merging redundant edges.
Through embodied simulations, we show that the proposed approach forms networks that are shorter than the shortest starlike trees and qualitatively closer to Steiner trees.
We further demonstrate its practical applicability through experiments with physical robots.

The remainder of this paper is organized as follows. Section~\ref{sec:related-work} reviews prior work on Steiner tree problems and multi-robot connectivity maintenance, followed by the problem formulation in Section~\ref{sec:problem-formulation}. Section~\ref{sec:methodology} then introduces the proposed decentralized network formation method, while Section~\ref{sec:analysis} analyzes its network-length properties and connectivity preservation. The approach is evaluated in simulation and on physical robots in Sections~\ref{sec:simulations} and~\ref{sec:experiments}, respectively. Finally, Section~\ref{sec:conclusion} concludes the paper.

\section{Related Work}
\label{sec:related-work}

\subsection{Steiner Trees and Relay Placement}
\label{sec:steiner-tree}

A Euclidean Steiner tree connects a set of terminals with minimum total network length by allowing additional points, called Steiner points, beyond the terminals themselves~\cite{promel2002Steiner}. Computing such a tree is NP-hard~\cite{brazil2014uniform}.
Classical Steiner trees impose no upper bound on edge length, whereas communication links between robots are typically limited by their communication range. The Steiner tree problem with a minimum number of Steiner points and bounded edge length therefore seeks the fewest additional relays required to connect a set of terminals while respecting a maximum edge length~\cite{lin1999steiner}. Algorithms for the bounded-edge Steiner problem~\cite{lin1999steiner,chen2017analysis,shin2023efficient} and relay-node placement in wireless sensor networks ~\cite{cheng2008relay,lloyd2007relay} centrally determine relay positions from a set of terminal locations; see~\cite{younis2008strategies} for a broader survey of node-placement strategies. 

Related robotic approaches have also used relay placement and Steiner-tree formulations. Yanmaz et al.~\cite{yanmaz2024dynamic} jointly planned the trajectories of search UAVs and dedicated relay UAVs while maintaining connectivity to a ground station. Nguyen et al.~\cite{nguyen2020new} considered target coverage and connectivity using a weighted Steiner-tree formulation, while Lee et al.~\cite{lee2022dynamic} used a central controller to add, remove, and relocate branches in a robot network for resource transport. Kent et al.~\cite{kent2026robust} proposed an architecture in which a base station computes a Semi-Steiner tree from global task locations and propagates the resulting topology through the swarm.

These methods can directly optimize relay placement or network geometry, but require globally available information and centralized computation, even when robot-level deployment is distributed. In contrast, we address the decentralized formation of a bounded-edge network, without any robot or external entity computing its complete structure.

\subsection{Connectivity-Preserving Control}

Distributed connectivity-maintenance methods regulate robot motion so that the communication graph remains connected while the robots perform other tasks. Krupke et al.~\cite{krupke2015distributed} proposed a cohesive swarm controller that maintains a redundant connected structure under external leader forces. Imran et al.~\cite{imran2016formal} used local virtual forces to improve connected coverage, while Luo et al.~\cite{luo2020behavior} combined task behaviors with global and subgroup connectivity maintenance. Connectivity constraints have also been enforced using control barrier functions~\cite{capelli2020connectivity}.

These approaches primarily treat connectivity as a constraint on robot motion or swarm cohesion. Although some methods select particular communication links or spanning structures to enforce connectivity, they do not generally optimize the geometric structure of a relay network, its total length, or the number of robots assigned as relays.

\subsection{Distributed Relay-Network Formation}

Other distributed approaches explicitly construct or modify relay-network topology while preserving connectivity. Majcherczyk et al.~\cite{Majcherczyk2018} deployed robots into starlike networks connecting multiple locations.
% , although the resulting structures are principally starlike. 
Varadharajan et al.~\cite{varadharajan2020swarm} formed and repaired relay chains between a base and a moving leader. Miyauchi et al.~\cite{miyauchi2022multi} formed a dynamic chain connecting two moving operators, enabling communication and robot exchange between them. Yi et al.~\cite{yi2021Distributed} inserted relay robots into the communication topology to enable task robots to reach their destinations. Soma et al.~\cite{soma2023complete} characterized connectivity-preserving local topology-manipulation operations for robot swarms.

% These works address path-graph or starlike deployment, task feasibility, or general graph reconfiguration. 
Although these works address different aspects of relay-network formation and reconfiguration, they do not jointly address the decentralized construction of a general multi-terminal tree, reduction of its total edge length, and removal of unnecessary relays. Our approach addresses this combination while maintaining connectivity and without centrally computing the final network topology.

\section{Problem Formulation}
\label{sec:problem-formulation}

Consider a planar environment with $n$ locations of interest $\mathbf{q}_1,\dots,\mathbf{q}_n \in \mathbb{R}^2$ and $m$ robots, where $m>n$.
Let $\mathcal{R}=\{1,\dots,m\}$ denote the set of robots, with $\mathbf{p}_i \in \mathbb{R}^2$ denoting the position of robot $i \in \mathcal{R}$.
We assume that each location of interest is assigned to a distinct robot.
All robots can move and communicate with all other robots within their communication range $d^{\text{com}}$ and have unique IDs.
The robots can autonomously establish communication links with one another, enabling the formation and maintenance of a communication network $G=(V,E)$, where $V\subseteq\mathcal{R}$ and $E\subseteq\bigl\{\{i,j\}\mid i,j \in V,\ i \neq j,\ \|\mathbf{p}_i -\mathbf{p}_j\|\leq d^{\text{com}}\bigr\}$.
Throughout this paper, $\|\cdot\|$ denotes the Euclidean norm.
The robots are initially clustered in a designated deployment region, sufficiently close to exchange messages.

The robots' objective is to form a network $G$ that
\begin{itemize}
    \item extends to the locations of interest; that is, for each location $\mathbf{q}_a$, $a\in\{1,\dots,n\}$, there exists  an observer robot $i_a\in V$ with $\mathbf{p}_{i_a}=\mathbf{q}_a$;
    % such that $\|\mathbf{p}_{i_a}-\mathbf{q}_a\|\leq \delta$;
    \item for any pair of distinct locations of interest $\mathbf{q}_a$ and $\mathbf{q}_b$, contains a path in $G$ connecting the corresponding observer robots $i_a$ and $i_b$;
    \item has minimum total edge length, $\sum_{\{i,j\}\in E} \|\mathbf{p}_i - \mathbf{p}_j\|$.
\end{itemize}
This formulation can be viewed as an embodied counterpart to the bounded-edge variant of the Euclidean Steiner tree problem, with at most $m$ robots available to form the network.

\section{Methodology}
\label{sec:methodology}

\begin{figure*}
    \centering
    \includegraphics[width=\linewidth]{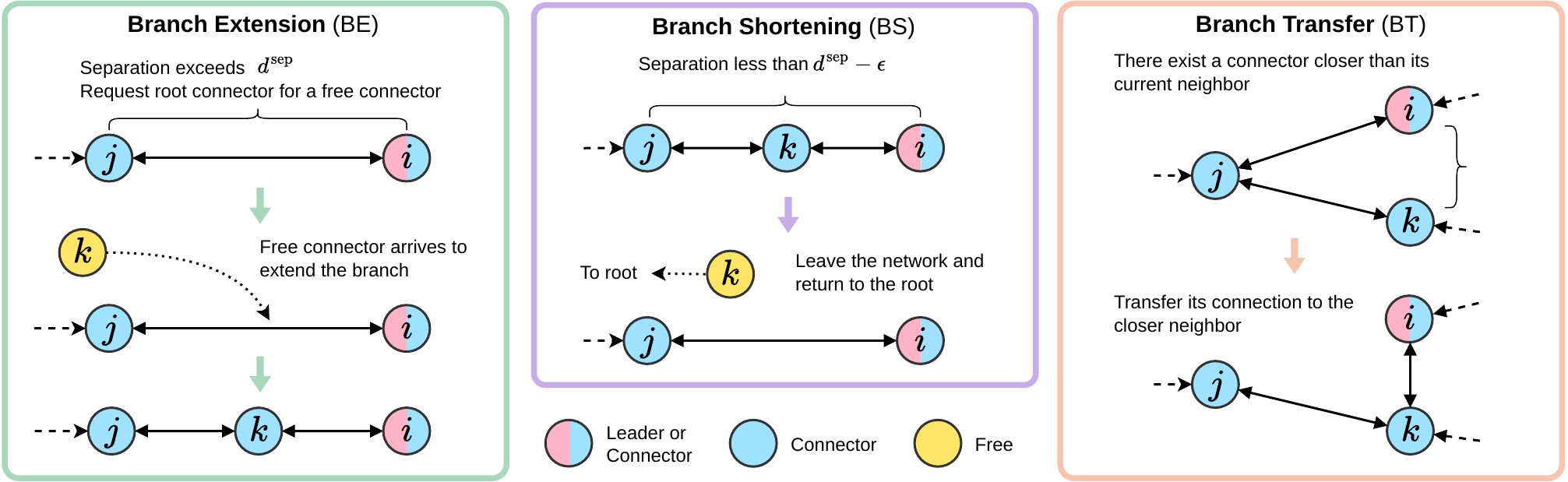}
    \caption{Overview of the three self-organized topology reconfiguration processes: branch extension, branch shortening, and branch transfer.}
    \label{fig:network-change}
\end{figure*}

The swarm consists of $n$ \textit{leader} robots and $m-n$ \textit{worker} robots.
Workers can either be a \textit{connector} if they are part of the network or \textit{free} otherwise.
Each leader robot is assigned a specific location of interest, hereafter referred to as its \textit{target}, and is responsible for reaching that location. We assume that each leader is capable of autonomously navigating to its assigned target using onboard sensing and self-motion estimation, such as odometry or path integration, without relying on GPS or a prior map of the environment.
First, we discuss the local operations to construct and refine the network, then implementation details and the robot's motion.

\subsection{Self-organized Topology Reconfiguration}

The swarm initially elects one worker to become the \textit{root} connector (e.g., by choosing the robot with the lowest unique ID).
All other workers are initially free and remain close to the root.
We assume that the root is initially within communication range of every leader.
The resulting communication graph is a tree, where the leaders represent leaf nodes.
Through hop-to-hop communication, the leaders and connectors forming the network each detect and update a list of hop counts from themselves to other leaders and the root.
Leaders and connectors monitor their distance to their \textit{parent} connector, defined as the unique network neighbor that has a smaller hop count to the root.

To allow leaders to move towards their respective target locations while maintaining connectivity to the root, free workers must switch to becoming a connector and extend the network.
We define three topology reconfigurations that the robots can locally apply to modify the network structure in a self-organized manner: (i) branch extension, (ii) branch shortening, and (iii) branch transfer (see Fig.~\ref{fig:network-change}).
A robot can take part in at most one topology reconfiguration at a time.
\subsubsection{Branch Extension (BE)}

When the distance to the parent connector exceeds the separation threshold $d^{\text{sep}}$, a free worker may join the network through branch extension.

Fig.~\ref{fig:network-change} (left) illustrates branch extension. Consider robot $i$, which may be either a leader or a connector, and let $j$ be its parent connector. 
We use $d_{ij} = \|\mathbf{p}_i - \mathbf{p}_j\|$ to denote the distance between robots $i$ and $j$.
Robot $i$ initiates branch extension if the following conditions are satisfied:
\begin{itemize}
    \item $BE_{\text{1}}$: $d_{ij} > d^{\text{sep}}$.
    \item $BE_{\text{2}}$: $i$ is not performing a topology reconfiguration.
\end{itemize}

When the above conditions are satisfied, robot $i$ informs the root, either directly or via hop messaging, that the connection between $i$ and $j$ needs to be extended.
Upon receiving this message, the root assigns a free worker $k$ to travel along the network and join in between $i$ and $j$.
Once robot $k$ reaches a position no more than $\mu$ away from the midpoint of $i$ and $j$, 
where $\mu \le \frac{d^{\text{sep}}}{2}$,
it sends a request to both robots to join the network.
If both $i$ and $j$ accept the request, $k$ is added to the network, and all three robots update their connections to their new neighbors.

\subsubsection{Branch Shortening (BS)}

Depending on the movement of the leaders and the sequence in which branch extensions and transfers occur, some connectors may end up being redundant.
To reduce the number of connectors required to maintain connectivity, a redundant connector may remove itself from the network through branch shortening.

Fig.~\ref{fig:network-change} (middle) illustrates branch shortening. Consider connector $k$, its parent connector $j$, and a robot $i$, which may be either a leader or a connector, adjacent to $k$.
Branch shortening removes connector $k$ from the network by replacing its connections to $i$ and $j$ with a direct connection between $i$ and $j$.
Robot $k$ initiates branch shortening if the following conditions are satisfied:
\begin{itemize}
    \item $BS_{\text{1}}$: $d_{ij} < d^{\text{sep}}-\epsilon$.
    \item $BS_{\text{2}}$: $k$ has exactly two neighboring robots.
    \item $BS_{\text{3}}$: $k$ is not performing a topology reconfiguration.
\end{itemize}
Here, $\epsilon > 0$ is a buffer that prevents oscillation between further branch extension and branch shortening due to small distance fluctuations.

When the above conditions are satisfied, $k$ sends a request to $i$ and $j$ to perform branch shortening.
If the request is accepted, $i$ and $j$ connect, after which $k$ is removed from $V$ and its incident edges from $E$.
Robot $k$ then becomes free and moves along the network toward the root.

\subsubsection{Branch Transfer (BT)}

Applying branch extension preserves the starlike tree structure connecting the root to the leaders.
To achieve structures beyond starlike trees, a leader or connector may transfer its connection to a connector that is geometrically closer through branch transfer.

Fig.~\ref{fig:network-change} (right) illustrates branch transfer. Consider robot $i$, which may either be a leader or a connector, a neighboring connector $j$, and a connector $k$ adjacent to $j$.
Branch transfer changes the connection $\{i,j\}$ to $\{i,k\}$. 
Robot $i$ initiates branch transfer if the following conditions are satisfied:
\begin{itemize}
    \item $BT_{\text{1}}$: $d_{ik} < d_{ij} - \epsilon$.
    \item $BT_{\text{2}}$: $j$ is connected to at least three neighboring robots.
    \item $BT_{\text{3}}$: $i$ is not performing a topology reconfiguration.
\end{itemize}
Here, $\epsilon > 0$ is a buffer similar to branch shortening to prevent triggering further branch transfers, switching back and forth between its old and new connections.

When the above conditions are satisfied, $i$ sends a request to $j$ and $k$ to perform branch transfer.
Condition $BT_{\text{2}}$ ensures that $j$ remains connected to at least two robots after the transfer.
% After the transfer, connector $k$ adjusts its position toward the geometric median of its new network neighbors. 
As shown later in Section~\ref{sec:analysis}, the edge replacement strictly decreases the total network length.
% , and this subsequent motion does not increase it.

\begin{figure}
    \centering
    \includegraphics[width=\linewidth]{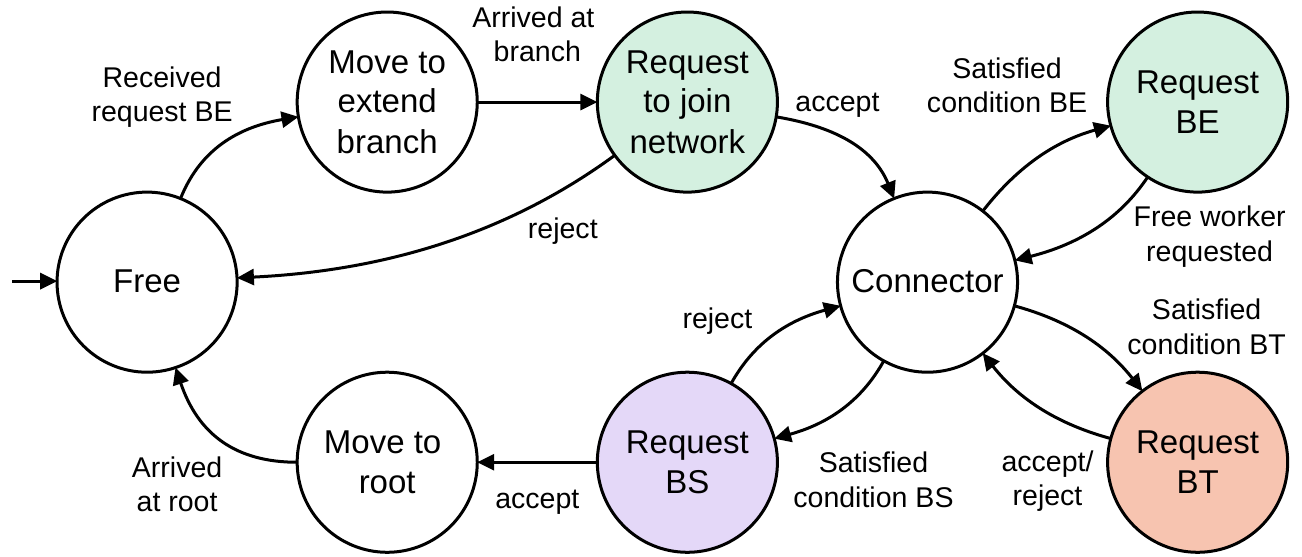}
    \caption{Finite state machine executed by each worker robot. The colored states indicate the different topology reconfiguration processes.}
    \label{fig:state-machine}
\end{figure}

The finite state machine shown in Fig.~\ref{fig:state-machine} describes the overall behavior of a worker robot during branch extension, branch shortening, and branch transfer.

\subsection{Robot Controller Implementation}

To model the aforementioned topology reconfigurations and robot motion at the level of a robot, we use supervisory control theory (SCT) which is a formal approach to control discrete event systems~\cite{ramadge1987Supervisory}.
Our approach is based on~\cite{lopes2016Supervisory}, which details the SCT fundamentals.
The descriptions of the models can be found in the multimedia material.
The source code will be made available at the time of publication.

\subsection{Robot Motion}

This section first describes how each connector moves to minimize its total distance to its network neighbors.
Together with the topology reconfiguration described above, this motion helps reduce the network's total edge length.
The section then describes how free workers move while waiting near the root and while traversing the network.

\subsubsection{Minimizing the Distance to Network Neighbors}

Connectors adjust their positions toward the geometric median of their neighbors in the network.
The \textit{geometric median} is a point that minimizes the sum of distances to a set of points in Euclidean space.
Formally, given a finite set of points $\mathcal{X}\subset\mathbb{R}^2$, the geometric median is defined as
\begin{equation}
    \underset{\mathbf{y}\in \mathbb{R}^2}{\arg\min} \sum_{\mathbf{x}\in\mathcal{X}} \|\mathbf{x}-\mathbf{y}\| .
\end{equation}
The geometric median may coincide with a point in $\mathcal{X}$.

Although closed-form solutions exist for sets of up to three points, the geometric median is generally computed using an iterative process.
In our approach, each connector computes a target position using the positions of its neighboring robots in the network.
The target position is computed by minimizing the sum of distances to these neighbors subject to a preferred safety threshold $d^\text{safe}$, where $d^{\text{sep}} < d^{\text{safe}} < d^{\text{com}}$.
If the constrained problem is infeasible, the solver returns the best-effort target position.

Let $\mathcal{N}_i^G=\{j\in V \mid \{i,j\}\in E\}$ denote the set of network neighbors of connector $i$.
The target position is obtained by solving
\begin{align}
    \underset{\mathbf{y}\in\mathbb{R}^2}{\text{minimize}} & \quad \sum_{j\in\mathcal{N}_i^G} \|\mathbf{p}_j-\mathbf{y}\| \\
    \text{subject to} & \quad \|\mathbf{p}_j-\mathbf{y}\| \leq d^{\text{safe}} \qquad \forall j\in\mathcal{N}_i^G.
\end{align}

Note that when all points are collinear, the geometric median is the middle point when the number of points is odd. When the number of points is even, any point between the two middle points is a geometric median; in this work, we use the midpoint between them.

Each connector computes the target position using the sequential least-squares quadratic programming~\cite{Kraft1988} algorithm, as implemented in the NLopt library~\cite{nlopt}.
The displacement toward the target position is capped at the maximum step length per time step.
Before moving, the SCT-based controller validates the proposed step against every network neighbor. 
The step is permitted only if each resulting distance either remains at most \(d^{\text{safe}}\) or, when its current value lies in \((d^{\text{safe}},d^{\text{com}}]\), decreases. 
Otherwise, the robot is prohibited from moving and remains stationary.
This safety mechanism ensures that connectors maintain their network connections.

\subsubsection{Flocking}

All free workers flock near the root using attractive and repulsive virtual forces from nearby robots, similar to~\cite{miyauchi2022multi}.
The virtual force experienced by robot $i$ is given by
\begin{equation}
    \mathbf{u}_i = \alpha\mathbf{u}^{\text{attract}}_i + \beta\mathbf{u}^{\text{repulse}}_i,
\end{equation}
where $\alpha$ and $\beta$ are positive scalars that weight the influence of each force component.
Let $\mathcal{N}_i$ denote the set of robots that are within $d^{\text{com}}$ of $i$ .

The attractive component $\mathbf{u}^{\text{attract}}_i$ pulls robot $i$ toward the root and is defined as
\begin{equation}
    \mathbf{u}^{\text{attract}}_i = \frac{1}{|\mathcal{N}^{\text{attract}}_i|}\sum_{j\in\mathcal{N}^{\text{attract}}_i}\frac{\mathbf{p}_j - \mathbf{p}_i}{\|\mathbf{p}_j - \mathbf{p}_i\|},
\end{equation}
where $\mathcal{N}^{\text{attract}}_i \subseteq \mathcal{N}_i$ 
contains only the root if it is in communication range. Otherwise, it contains neighboring free workers with fewer communication hops to the root than $i$.
If $\mathcal{N}^{\text{attract}}_i=\emptyset$, then $\mathbf{u}^{\text{attract}}_i=\mathbf{0}$.

The repulsive component $\mathbf{u}^{\text{repulse}}_i$ pushes robot $i$ away from nearby robots and is defined as
\begin{equation}
    \mathbf{u}^{\text{repulse}}_i = -\frac{1}{|\mathcal{N}_i|}\sum_{j\in\mathcal{N}_i}\frac{\sigma^\lambda}{\|\mathbf{p}_j - \mathbf{p}_i\|^\lambda}\frac{\mathbf{p}_j - \mathbf{p}_i}{\|\mathbf{p}_j - \mathbf{p}_i\|},
\end{equation}
where $\sigma$ is the desired separation between robots, and $\lambda$ is the exponent controlling the decay of the repulsive force.
If $\mathcal{N}_i=\emptyset$, then $\mathbf{u}^{\text{repulse}}_i=\mathbf{0}$.

\subsubsection{Traversing the Network}
A free worker traverses along the network to reach its assigned target edge during branch extension or to return toward the root after branch shortening.
Connectors continuously share, with any robot within their communication range, their hop counts to the root and leaders. 
A traversing robot uses this information to choose which nearby connector to move toward.
During branch extension, the free worker is assigned a target edge, the connection between two robots, and a leader whose path from the root passes through that edge.
It follows connectors with decreasing hop counts to the assigned leader until it reaches the target edge.
After branch shortening, it instead moves toward the neighboring connector with the decreasing hop counts to the root.

\section{Analysis}
\label{sec:analysis}

Let the length of a network \(G=(V,E)\) be
\begin{equation}
	L(G) = \sum_{\{i,j\}\in E} d_{ij}.
\end{equation}

\begin{theorem}
\label{thm:connectivity-preservation}
Suppose the initial network comprising the leaders and the root node is connected, and that all nodes of the network sense and act sequentially. %Under the assumption of sequential motion updates, t
The network remains connected under the proposed controller.
\end{theorem}

\begin{proof}
During motion, the SCT model permits a network node to move only when its resulting distance to the current position of each network neighbor either remains within \([0,d^{\text{safe}}]\) or decreases\footnote{Note that the focal robot can determine whether these properties hold given its neighboring robots can be considered stationary due to the sequential updating assumption.}, if lying in \((d^{\text{safe}},d^{\text{com}}]\).
In either case, the resulting distance remains within \(d^{\text{com}}\), so every existing network edge remains a feasible communication link after each permitted motion step.

It remains to consider the self-organized reconfigurations of the network topology. 
Branch extension adds a node \(k\) to \(V\) and replaces an edge \(\{i,j\}\) with edges \(\{i,k\}\) and \(\{k,j\}\), 
provided that \(\mu \le \frac{d^{\text{sep}}}{2}\); 
hence, it replaces one edge with a path between the same endpoints. 
Branch shortening replaces edges \(\{i,k\}\) and \(\{k,j\}\) with an edge \(\{i,j\}\). By \(BS_{\text{1}}\), \(d_{ij}<d^{\text{sep}}-\epsilon<d^{\text{com}}\), so this replacement edge is valid; by \(BS_{\text{2}}\), \(k\) has no other incident edges to disconnect from, and thus it becomes free. 
Finally, for branch transfer, \(BT_{\text{1}}\) gives \(d_{ik}<d_{ij}-\epsilon < d_{ij} \leq d^{\text{com}}\), so the replacement edge \(\{i,k\}\) is valid. Because \(k\) remains adjacent to \(j\), this replacement reconnects the component containing \(i\) to the component containing \(j\), and \(BT_{\text{2}}\) ensures that \(j\) retains at least two neighbors. 
Thus, every completed topology reconfiguration preserves a path between all remaining robots. As both permitted motion and topology reconfigurations preserve connectivity, the initially connected network remains connected at all times.
\end{proof}

\begin{theorem}
    \label{thm:without-branch-transfer-starlike}
    Assume point robots with no inter-robot repulsion. If sufficient number of robots are present and branch transfer is disabled or does not occur, the %stabilized 
    network converges to a shortest starlike tree connecting the target locations. 
\end{theorem}

\begin{proof}
Initially, the \(n\) leaders are adjacent to the root, thus forming a starlike graph with the root as the central vertex. The leaders move respectively toward distinct target locations \(\mathbf{q}_1,\dots,\mathbf{q}_n\). When a leader or connector reaches the separation limit \(d^{\text{sep}}\) from its parent, movement stalls until branch extension inserts a connector near the midpoint of the corresponding edge. 
Because each branch extension increases the maximum length of a branch by at least \(d^{\text{safe}}\), sufficiently many robots at the root ensure that every leader eventually reaches its target.
As branch transfer does not occur, each extension only 
concerns one of the $n$ path graphs that connect the root with the leaders, and the path graph property is retained.
Therefore, the graph remains starlike throughout the process, with one branch from the root to every leader.
With sufficient number of robots available, all branches extend to every target.

Every non-root connector has exactly two network neighbors. The geometric median of two points is any point on the line segment joining them. Thus, once the network stabilizes, every non-root connector lies on the segment joining its two branch neighbors. Applying this observation along each root-to-target branch shows that all of its connectors lie on the segment from the root to that target; hence, the root, target, and all connectors on each branch are collinear.

Let \(u_i\) denote the robot that is a neighbor of the root on the branch ending at target \(\mathbf{q}_i\). As all robots on that branch positioned themselves along the line segment connecting the root to the target, 
the root's geometric-median rule over \(\{\mathbf{p}_{u_1},\dots,\mathbf{p}_{u_n}\}\) is therefore equivalent to the geometric-median rule over \(\{\mathbf{q}_1,\dots,\mathbf{q}_n\}\). Hence, the root also places itself at a geometric median of the target locations.

Because every branch is straight, its length equals the distance from the root to its target. Hence, \(L(G)\) equals the sum of the distances from the root to all targets, which by definition is minimized at their geometric median. Therefore, \(G\) is a shortest starlike tree connecting the targets.
\end{proof}

\begin{theorem}
\label{thm:branch-transfer-length}
Consider a branch transfer initiated by robot \(i\), which replaces its network connection to connector \(j\) with a connection to connector \(k\), where \(k\) is adjacent to \(j\). Suppose that, after the transfer, \(k\) moves in a straight line towards the geometric median of its network neighbors, while all other nodes remain stationary.
Then the resulting network \(G'\) has strictly smaller length than the original network \(G\).
\end{theorem}

\begin{proof}
Let \(G^{+}\) denote the network immediately after the edge replacement and before \(k\) moves. Branch transfer replaces the edge \(\{i,j\}\) with \(\{i,k\}\), while leaving all other network edges unchanged (see Fig.~\ref{fig:network-change} (right)). Thus,
\begin{equation}
    E^{+} = \bigl(E \setminus \{\{i,j\}\}\bigr) \cup \{\{i,k\}\}.
\end{equation}
By condition \(BT_{\text{1}}\), the transfer is accepted only when \(d_{ik}<d_{ij}-\epsilon\). Therefore,
\begin{align}
    L(G^{+})
    &= L(G) - d_{ij} + d_{ik} < L(G) -\epsilon < L(G).
\end{align}

Let \(\mathcal{N}_k^{+}\) denote the set of network neighbors of \(k\) immediately after the transfer.
Let \(\mathbf{p}_k\) be the position of \(k\) before its motion, and let \(\mathbf{p}_k^{*}\) be the geometric median of the fixed positions of the robots in \(\mathcal{N}_k^{+}\). By the definition of the geometric median,
\begin{equation}
    \sum_{\ell\in\mathcal{N}_k^{+}}
    \|\mathbf{p}_k^{*}-\mathbf{p}_\ell\|
    \leq
    \sum_{\ell\in\mathcal{N}_k^{+}}
    \|\mathbf{p}_k-\mathbf{p}_\ell\|.
\end{equation}
The sum of distances to the fixed neighbors is a convex function minimized at \(\mathbf{p}_k^{*}\), so it cannot increase as \(k\) moves along the line toward this point.
All other edge lengths remain unchanged, giving \(L(G')\leq L(G^{+})<L(G)\).
\end{proof}

\section{Simulations}
\label{sec:simulations}

\begin{figure*}[t]
    \centering
    % --- Left Side: 2x3 Grid ---
    \begin{minipage}[b]{0.6\textwidth}
        \centering
        \includegraphics[width=0.31\textwidth]{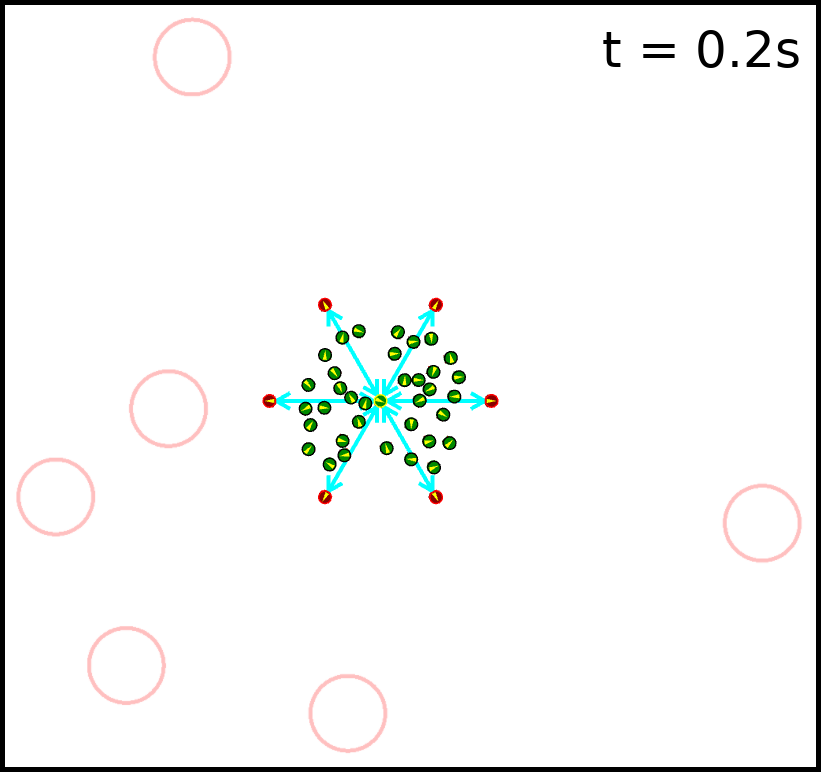} \hfil
        \includegraphics[width=0.31\textwidth]{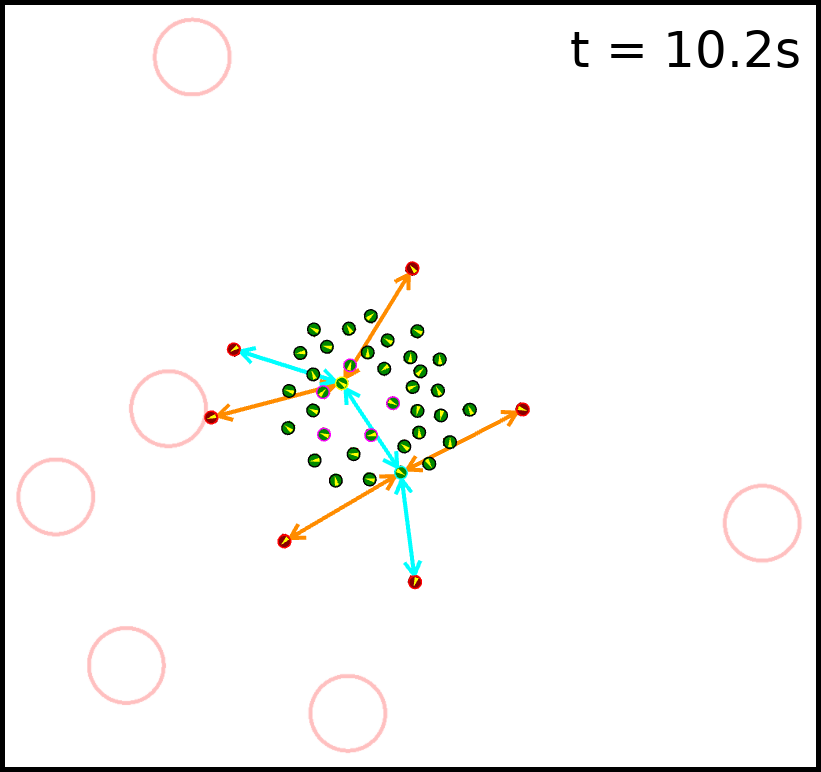} \hfil
        \includegraphics[width=0.31\textwidth]{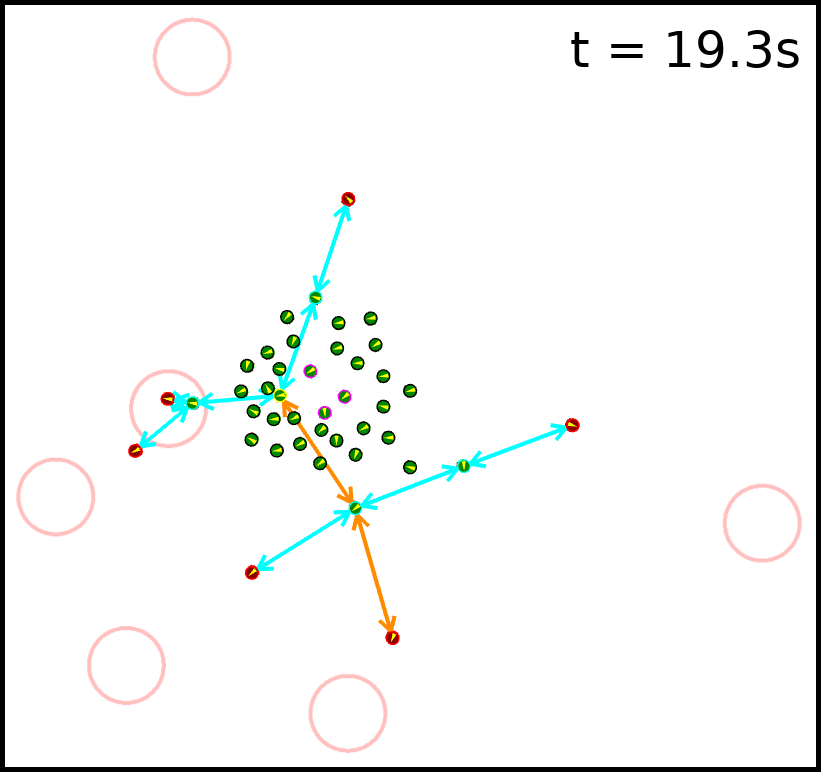}
        
        \vspace{0.1cm} 
        
        \includegraphics[width=0.31\textwidth]{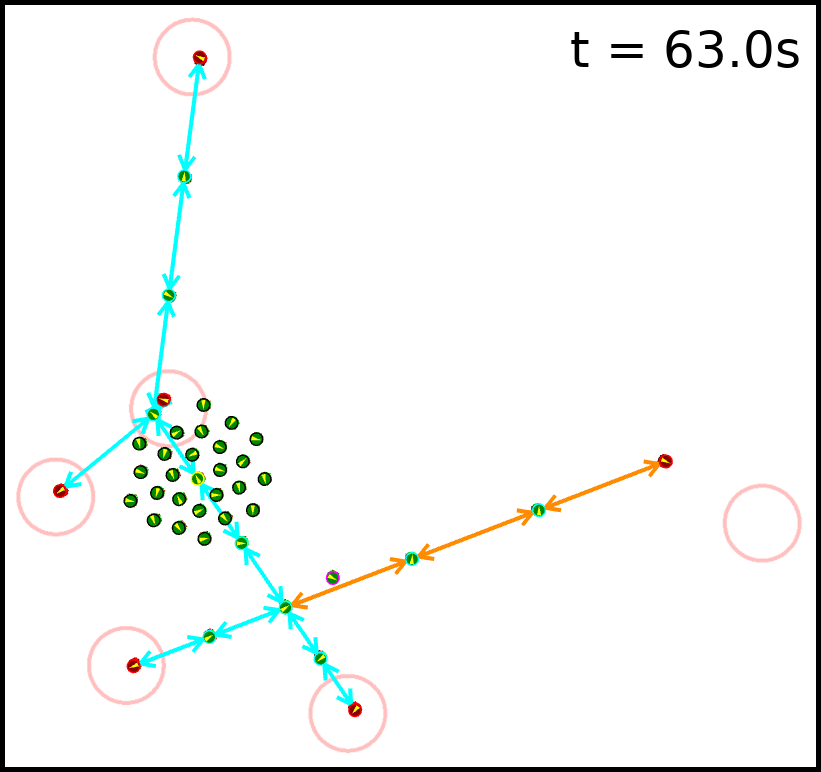} \hfil
        \includegraphics[width=0.31\textwidth]{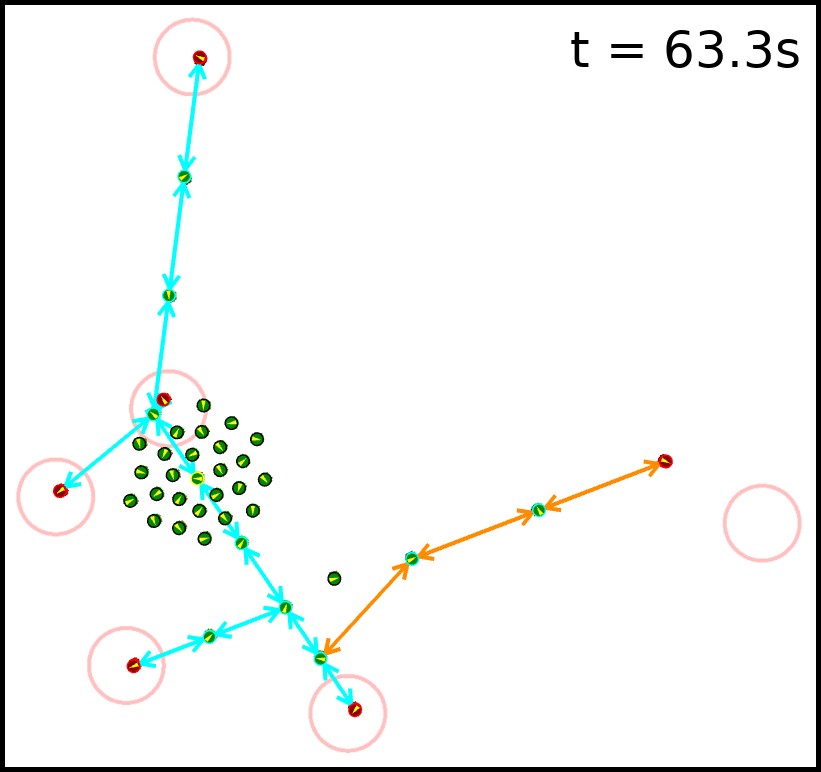} \hfil
        \includegraphics[width=0.31\textwidth]{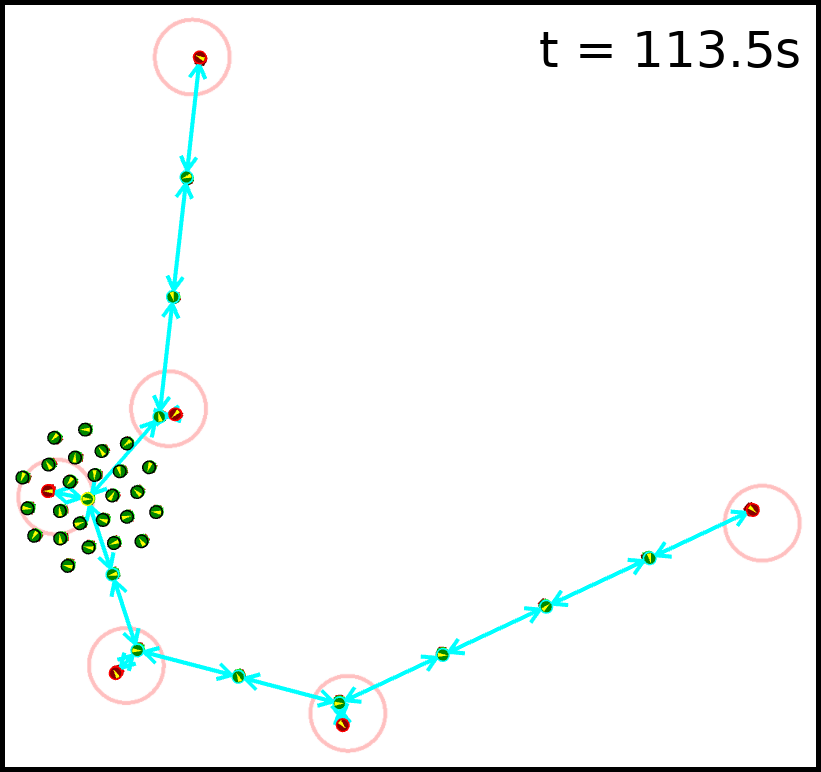}
    \end{minipage}
    % \hfill
    % --- Right Side: Single Large Figure ---
    \begin{minipage}[b]{0.35\textwidth}
        \centering
        \includegraphics[width=\textwidth]{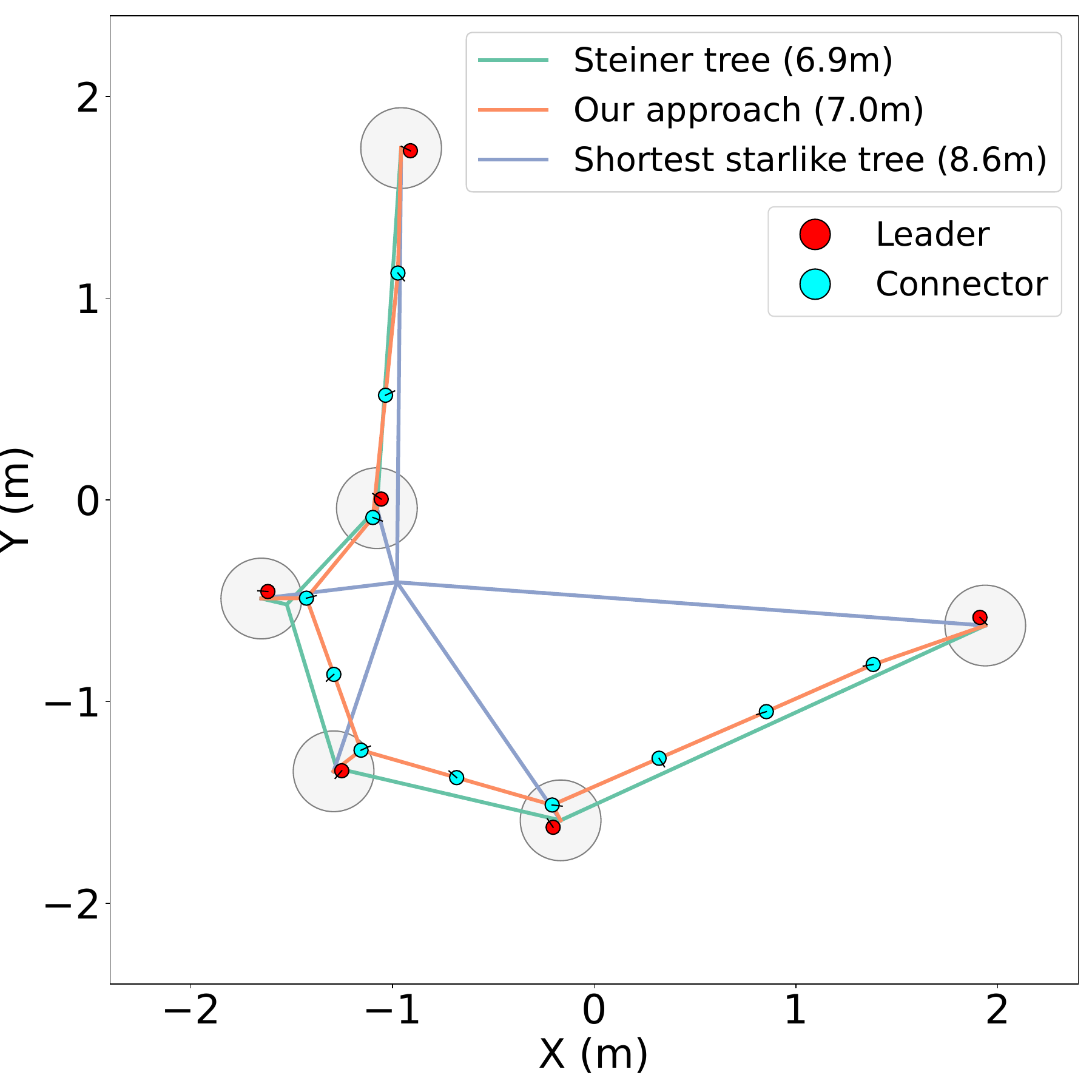}
    \end{minipage}

    \caption{Left: Snapshots from a simulation trial. Right: Final robot positions for the same trial. Gray circles indicate target locations. Lines show the network formed by our approach, the centrally computed Steiner tree, and the centrally computed shortest starlike tree, as indicated in the legend.}
    \label{fig:simulation-snapshots}
\end{figure*}

\subsection{Simulation Setup}

Simulation trials were conducted using ARGoS~\cite{pinciroli2012ARGoS} with the e-puck~\cite{mondada2009epuck}, a mobile differential-wheeled robot.
Robots were equipped with a range-and-bearing system for communication with neighbors, using a range of $d^{\text{com}} =$ 0.80\,m, a preferred safety threshold of $d^{\text{safe}} =$ 0.70\,m, and a separation threshold of $d^{\text{sep}} =$ 0.68\,m.
Target locations were uniformly randomly distributed in a circular arena of radius 2.4\,m.

The simulation physics and the robot control cycle were updated at 30\,Hz.
Trials were terminated 10\,s after all leaders were less than 0.05\,m from their target locations, allowing the robots to stabilize, or after 10 minutes, whichever occurred first.
Trials were repeated 100 times for each configuration.
Videos of selected trials can be found in the multimedia material.

\subsection{Network Length}

To evaluate the length of the networks formed by our approach, we compare them against two centrally computed baselines and one ablation:
\begin{itemize}
    \item \textbf{Centrally computed Steiner trees.} As described in Section~\ref{sec:steiner-tree}, Steiner trees are networks of minimal length. We use the GeoSteiner solver~\cite{winter1985algorithm,juhl2018GeoSteiner} with the target locations as input.
    \item \textbf{Centrally computed shortest starlike trees.} These have a root node at a geometric median of the target locations, computed by the iterative solver in~\cite{vardi2000multivariate}. The targets are connected to the root node by line segments.
    \item \textbf{Our approach without branch transfer.} For these, Theorem~\ref{thm:without-branch-transfer-starlike} predicts that the resulting network converges to the centrally computed shortest starlike tree.
\end{itemize}
We investigate the performance for $n\in\{2,\dots,10\}$ target locations.
The swarm consisted of $n$ leaders and $6n$ workers.
Each leader was assigned a unique target location to visit.

In all trials, the leaders reached their target locations while the workers maintained the network. 
Fig.~\ref{fig:simulation-snapshots} (left) illustrates the progression of one such trial through a sequence of snapshots. The final network formed by our approach is shown in 
Fig.~\ref{fig:simulation-snapshots} (right), overlaid with the two centrally computed benchmarks for comparison; note that the latter disregard communication range constraints. 
One can observe that the shortest starlike tree features a dedicated branch per target, extending from a root node, whereas our approach and the Steiner tree share branches among certain targets, reducing the total network length.

Unlike the sequential motion assumed in Theorem~\ref{thm:connectivity-preservation}, the implementation allows simultaneous robot motion. Although the theorem does not cover this, no connectivity loss was observed in any of the simulation trials.

\begin{figure}[t]
    \centering
    \includegraphics[width=\linewidth]{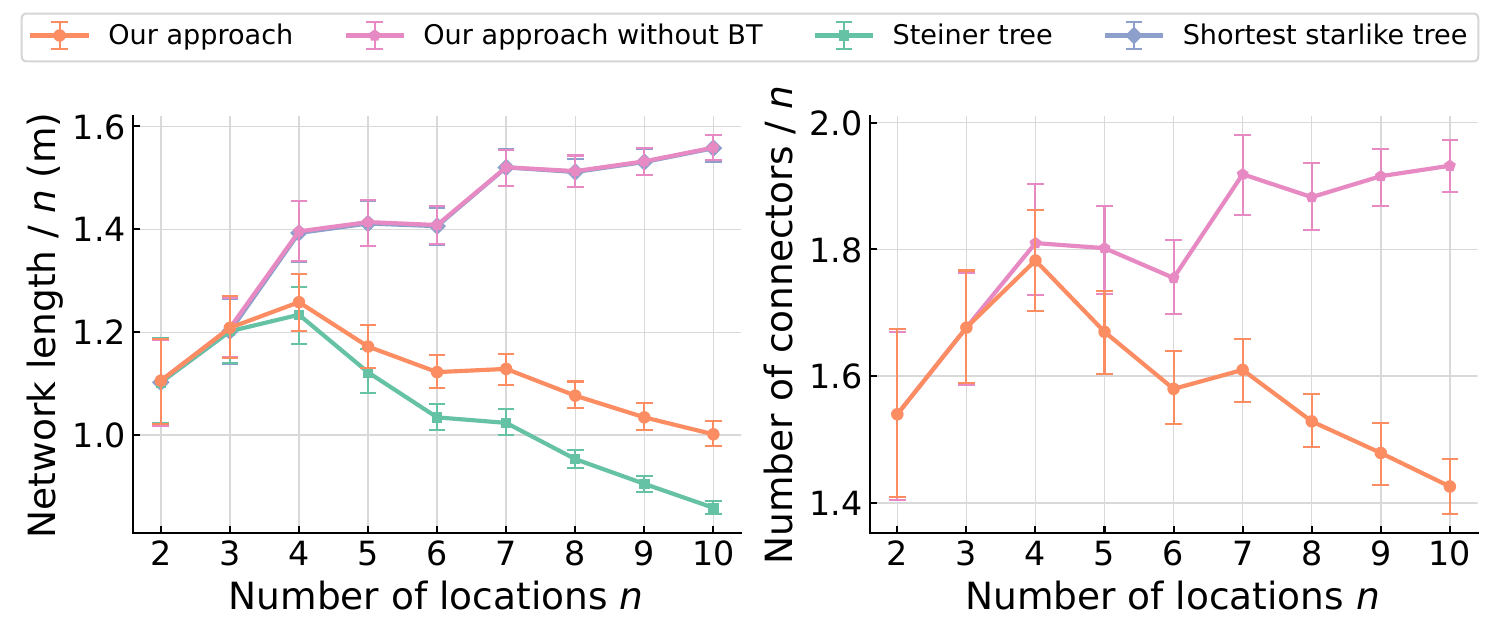}
    \caption{Comparison of the network length ($L(G)/n$) and number of connectors ($N_{\text{connector}}/n$) per target for varying numbers of target locations.}
    \label{fig:sim-result}
\end{figure}

Fig.~\ref{fig:sim-result} (left) shows the network length per target location for varying $n$. 
For $n=2$ and $3$, all approaches produce comparable network lengths.
As $n$ increases, the results begin to diverge: the shortest starlike tree sees a moderate increase in per-location length, whereas our approach and the Steiner tree also increase up to $n=4$ before exhibiting a clear downward trend. 
This decrease can hence be attributed to branch transfer, which enables segments to be shared among multiple leaders, reducing the per-location network length as $n$ grows.
Our approach without branch transfer mirrors the centrally computed shortest starlike tree, 
further confirming the effectiveness of the distributed implementation.

\subsection{Number of Connectors}

Fig.~\ref{fig:sim-result} (right) shows the number of connectors per target location for our approach with and without branch transfer. 
For $n\leq4$, both approaches require a comparable number of connectors per location. 
As $n$ increases, our approach with branch transfer exhibits a downward trend, as shared segments reduce the total number of connectors needed to maintain the network. 
Without branch transfer, the number of connectors per location remains relatively stable, consistent with the shortest starlike tree where each target requires a dedicated branch.
This confirms that branch transfer is effective not only in reducing network length but also in reducing the number of connectors deployed.

\begin{figure}[t]
    \centering
    \includegraphics[width=\linewidth]{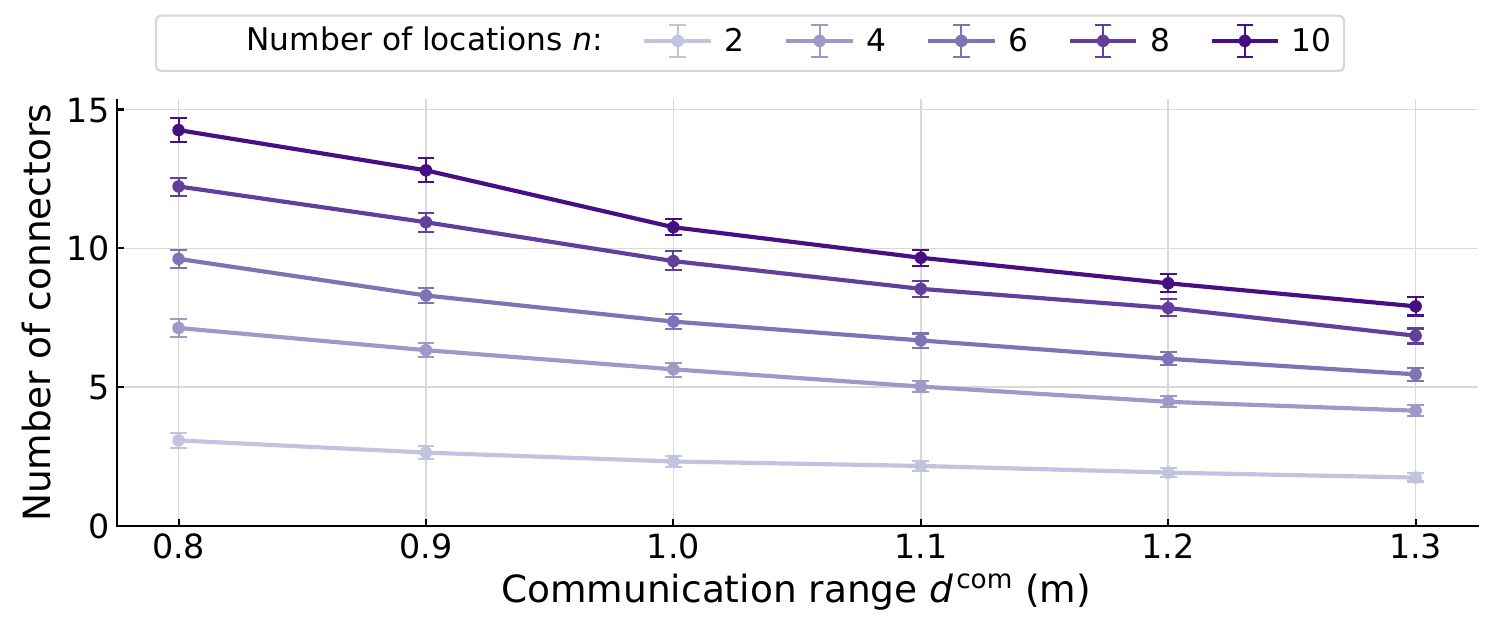}
    \caption{Number of connectors for varying communication ranges when using our approach.}
    \label{fig:result-comm-range}
\end{figure}

\subsection{Impact of Communication Range}

We further evaluated the impact of communication range by testing $d^{\text{com}} \in \{0.8, 0.9, \dots, 1.3\}$\,m for $n \in \{2, 4, 6, 8, 10\}$ target locations. 
Fig.~\ref{fig:result-comm-range} shows that as $d^{\text{com}}$ increases, the number of connectors decreases consistently across all values of $n$. 
This confirms that the approach generalizes across communication ranges, with robots effectively exploiting longer ranges to reduce the number of connectors needed to maintain the network.

\subsection{Moving Leaders}

To test whether the network remains connected under continuously moving leaders, we use four leaders and $d^{\text{com}}=0.8$.
The video in the multimedia material shows each leader visiting two target locations in sequence.
As the leaders move, the robots maintaining the network continue to adapt by applying the three topology reconfiguration operations.

\section{Experiments}
\label{sec:experiments}

\subsection{Experimental Setup}

\begin{figure}[t]
    \centering
    \subfloat[\label{fig:real-robots:with-BT}Our approach]{\includegraphics[width=.49\linewidth]{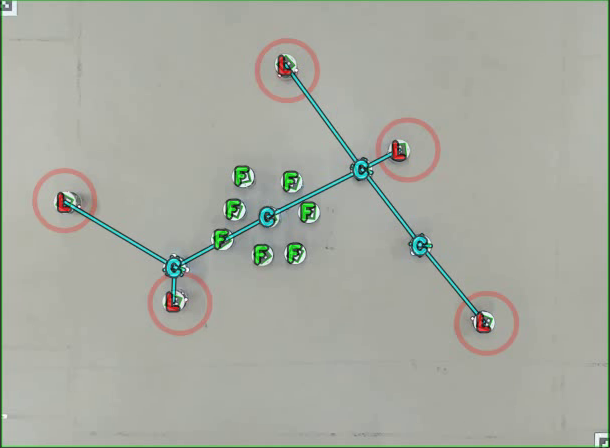}}
    \hfill
    \subfloat[\label{fig:real-robots:without-BT}Our approach without BT]{\includegraphics[width=.49\linewidth]{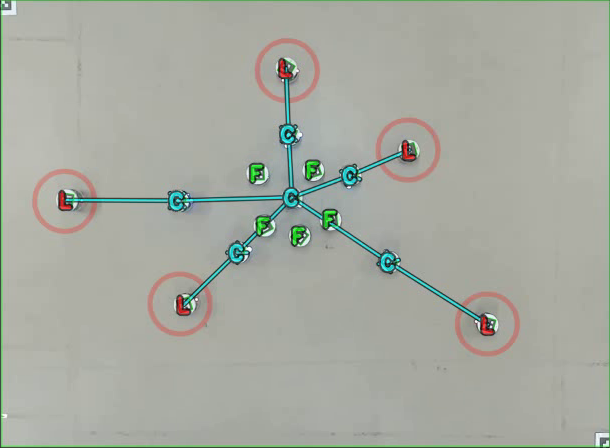}}
    \caption{Snapshots from a real robot experiment. 
    (a) Final network with branch transfer. (b) Final network without branch transfer. Leaders (L, red), connectors (C, cyan), and free workers (F, green) are labelled. Cyan lines denote connections between robots. Red circles indicate target locations.}
    \label{fig:real-robots}
\end{figure}

\begin{figure}[t]
    \centering
    \includegraphics[width=\columnwidth]{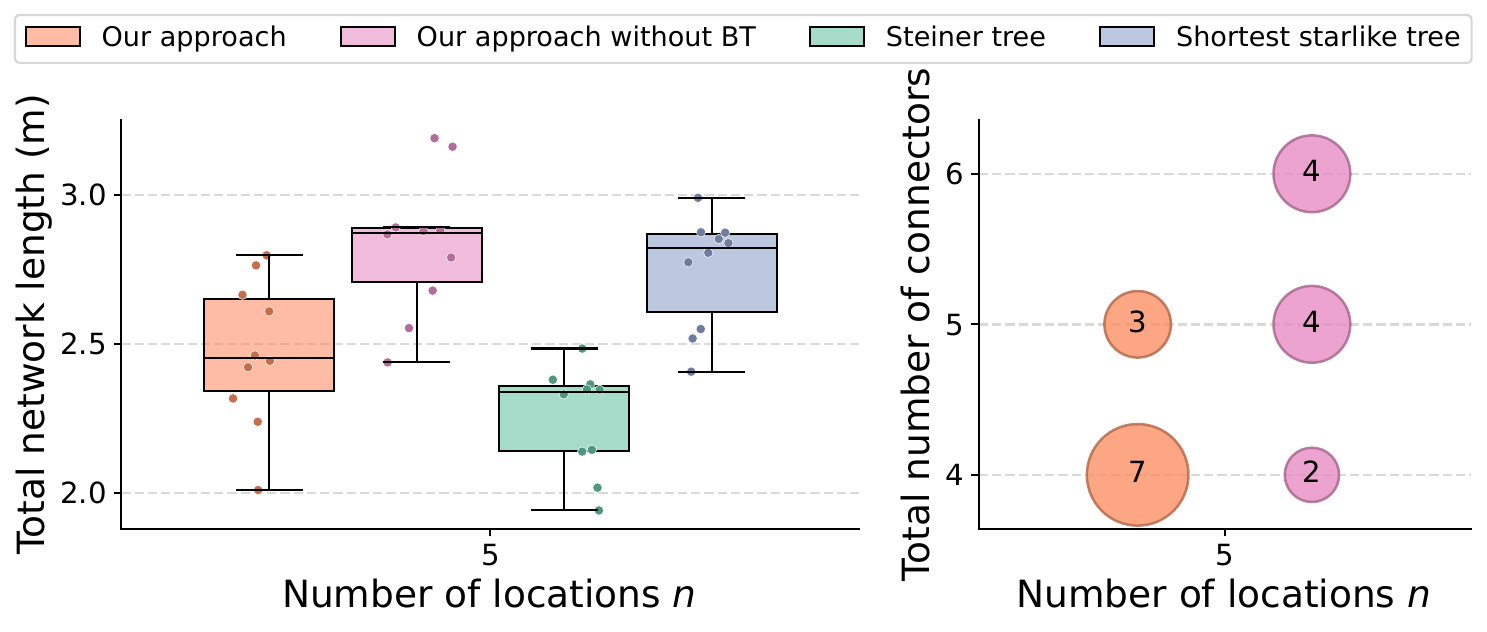}
    \caption{Comparison of the total network length and number of connectors from real robot experiments.}
    \label{fig:real-robots-results}
\end{figure}

We conducted experiments using 16 Pi-pucks~\cite{millard2017Pipuck} in a 2\,m$\times$1.5\,m arena to validate our approach in a real-world setting. 
The arena contained five uniformly distributed target locations. The swarm consisted of 5 leaders and 11 workers, with each leader assigned a unique target location. 
The robots were initially arranged in a 4\,$\times$\,4 grid at the center of the arena with random orientations. 
An overhead camera tracked the position and orientation of each robot via unique ArUco markers~\cite{millard2018ardebug}, and was used to implement a virtual range-and-bearing sensor.
We used $d^\text{com}=0.50$\,m, $d^\text{safe}=0.45$\,m, and $d^\text{sep}=0.43$\,m for the experiments. 
Trials were performed 10 times both with and without branch transfer. 
Each trial was terminated once all leaders were less than 0.03\,m from their target locations and all connectors ceased motion. 

\subsection{Results}

In all trials, the leaders
reached their targets, while workers maintained network connectivity. Videos are provided in the multimedia material. Fig.~\ref{fig:real-robots} shows snapshots of the final network formed with and without branch transfer.
Branch transfer reduced the total length from 2.88\,m to 2.46\,m and the number of connectors from six to four.
As shown in Fig.~\ref{fig:real-robots-results}, these experiments followed the trends observed in simulation: our approach formed networks that were significantly shorter than starlike trees and, averaging only 9.9\% longer than Steiner trees, while using fewer connectors.

\section{Conclusion}
\label{sec:conclusion}

We presented a decentralized controller that enables robot swarms to form and maintain communication networks among multiple locations using local sensing and communication, without global position information or centralized topology computation. 
The controller combines branch extension, branch shortening, and branch transfer with local motion toward the geometric median of neighboring robots to reduce total network length while preserving connectivity. 
Our analysis shows that the connectivity of the network is preserved,
presents an upper bound of its length which corresponds to that of shortest star-like trees, and
shows that every branch transfer strictly decreases network length. 
Simulations and physical-robot experiments reveal that branch transfer produces networks significantly shorter than the shortest starlike tree baseline and not far in length to centrally computed Steiner trees, while using fewer connectors required. 
Future work will address obstacle-rich environments and robustness to robot and communication failures.

\addtolength{\textheight}{0cm}   % This command serves to balance the column lengths
                                  % on the last page of the document manually. It shortens
                                  % the textheight of the last page by a suitable amount.
                                  % This command does not take effect until the next page
                                  % so it should come on the page before the last. Make
                                  % sure that you do not shorten the textheight too much.

%%%%%%%%%%%%%%%%%%%%%%%%%%%%%%%%%%%%%%%%%%%%%%%%%%%%%%%%%%%%%%%%%%%%%%%%%%%%%%%%

%%%%%%%%%%%%%%%%%%%%%%%%%%%%%%%%%%%%%%%%%%%%%%%%%%%%%%%%%%%%%%%%%%%%%%%%%%%%%%%%

%%%%%%%%%%%%%%%%%%%%%%%%%%%%%%%%%%%%%%%%%%%%%%%%%%%%%%%%%%%%%%%%%%%%%%%%%%%%%%%%
% \section*{APPENDIX}

% Appendixes should appear before the acknowledgment.

\section*{ACKNOWLEDGMENT}

This work was supported in part by OpenSwarm Project (EU's Horizon Europe Framework Programme under Grant 101093046, and UKRI under Grant 10048272 of the U.K. Government's Horizon Europe funding Guarantee), and Robotics Institute Germany BMBF under Grant 16ME1001.

%%%%%%%%%%%%%%%%%%%%%%%%%%%%%%%%%%%%%%%%%%%%%%%%%%%%%%%%%%%%%%%%%%%%%%%%%%%%%%%%

% References are important to the reader; therefore, each citation must be complete and correct. If at all possible, references should be commonly available publications.

\bibliographystyle{IEEEtran}
\bibliography{references}

\end{document}